\documentclass[letterpaper]{article} 
\usepackage{aaai2026}
\usepackage{times}  
\usepackage{helvet}  
\usepackage{courier}  
\usepackage[hyphens]{url}  
\usepackage{graphicx} 
\usepackage{natbib}  
\usepackage{caption} 
\usepackage{algorithm}
\usepackage{algorithmic}

\usepackage{pifont}

\usepackage{booktabs}
\usepackage{multirow}
\usepackage{amsmath}
\usepackage{amssymb}

\usepackage{newfloat}
\usepackage{listings}
\DeclareCaptionStyle{ruled}{labelfont=normalfont,labelsep=colon,strut=off} 
\floatstyle{ruled}
\newfloat{listing}{tb}{lst}{}
\floatname{listing}{Listing}
\usepackage{times}
\usepackage{helvet}
\usepackage{courier}
\usepackage{xcolor}
\usepackage{booktabs}
\usepackage{multirow}
\usepackage{amsmath}
\usepackage{natbib}
\usepackage{amsthm}
\usepackage{amssymb}
\usepackage{graphicx}
\usepackage{subcaption} 

\newtheorem{proposition}{Proposition}[section]

\title{DPRM: A Dual Implicit Process Reward Model in Multi-Hop Question Answering}
\author {
    Xinyi Wang\textsuperscript{\rm 1},
    Yiping Song\textsuperscript{\rm 1},
    Zhiliang Tian\textsuperscript{\rm 1 \thanks{Corresponding authors}},
    Bo Liu\textsuperscript{\rm 2 \footnotemark[1]},
    Tingjin Luo\textsuperscript{\rm 1},
    Minlie Huang\textsuperscript{\rm 3}
}
\affiliations {
    \textsuperscript{\rm 1}National University of Defense Technology\\
    \textsuperscript{\rm 2}Academy of Military Sciences\\
    \textsuperscript{\rm 3}The CoAI Group, Department of Computer Science and Technology, Tsinghua University\\
    songyiping@nudt.edu.cn, kyle.liu@nudt.edu.cn
}

\begin{document}

\maketitle

\begin{abstract}
In multi-hop question answering (MHQA) tasks, Chain of Thought (CoT) improves the quality of generation by guiding large language models (LLMs) through multi-step reasoning, and Knowledge Graphs (KGs) reduce hallucinations via semantic matching. Outcome Reward Models (ORMs) provide feedback after generating the final answers but fail to evaluate the process for multi-step reasoning. Traditional Process Reward Models (PRMs) evaluate the reasoning process but require costly human annotations or rollout generation. While implicit PRM is trained only with outcome signals and derives step rewards through reward parameterization without explicit annotations, it is more suitable for multi-step reasoning in MHQA tasks. However, existing implicit PRM has only been explored for plain text scenarios. When adapting to MHQA tasks, it cannot handle the graph structure constraints in KGs and capture the potential inconsistency between CoT and KG paths. To address these limitations, we propose the DPRM (Dual Implicit Process Reward Model). It trains two implicit PRMs for CoT and KG reasoning in MHQA tasks. Both PRMs, namely KG-PRM and CoT-PRM, derive step-level rewards from outcome signals via reward parameterization without additional explicit annotations. Among them, KG-PRM uses preference pairs to learn structural constraints from KGs. DPRM further introduces a consistency constraint between CoT and KG reasoning steps, making the two PRMs mutually verify and collaboratively optimize the reasoning paths. We also provide a theoretical demonstration of the derivation of process rewards. Experimental results show that our method outperforms 13 baselines on multiple datasets with up to $16.6\%$ improvement on Hit@1.

\end{abstract}


\section{Introduction}
In the Multi-hop Question Answering (MHQA) task, Chain of Thought (CoT) helps large language models (LLMs) to generate answers by thinking step by step, and structured knowledge \cite{Liu2025Causal, LiSARATRX25} like Knowledge Graphs (KGs) makes reasoning more accurate by matching facts. These KG-enhanced LLM reasoning methods achieve good results in many domains, including customer service \cite{Xu_2024custom}, finance \cite{nishat2025finance}, medical inquiry \cite{zhao2024medical,wu2025medical}, and other scenarios.

Most of the existing MHQA methods based on LLMs and KGs follow the paradigm of Graph Retrieval Augmented Generation (GraphRAG) \cite{mavromatis2024gnn, he2025g, edge2024graphrag, hu2025grag, wang-etal-2025-dcmkc}. They retrieve query-related knowledge from KG to assist LLM in generating reasoning paths and answers. Retrieval is usually completed before LLM reasoning, and the retrieved content is fixed. Once the retrieval results are incomplete or contain redundant information, LLM has difficulty acquiring new knowledge for step-level error correction. Other works make the LLMs interact with the KG to generate answers step by step \cite{sanmartin2024kg, li2024graphreader, jin2024graph, dong2024effiqa}. In interactive systems, the planning and reflection processes rely on the LLM's ability to evaluate reasoning paths, but this evaluation is typically qualitative rather than providing quantitative scores. Additionally, KG retrievers optimize for shortest query-answer paths but lack path quality evaluation. As a result, error paths will lead to wrong answers. Therefore, a quantitative evaluation mechanism is needed to evaluate and filter CoT steps and their corresponding KG triples during reasoning.

Meanwhile, in mathematical reasoning tasks \cite{huziyi2025}, there is a widely used and proven effective quantitative evaluation mechanism for reasoning—--reward models (RMs). Early RMs are mostly explicit mechanisms. The outcome reward models (ORMs) assign rewards to the entire response and provide feedback after generating the final answer \cite{christiano2017deep, liu2024skywork, ouyang2022training}. It can effectively replace manual reward design and make the model learn human preferences. However, ORMs fail to evaluate the reasoning process for complex tasks requiring multi-step reasoning. In addition, since rewards are emitted only at the terminal state, the sparse supervision signals struggle with stability and efficiency in training \cite{cao2024train1,chan2024train2}. Process reward models (PRMs) break through this limitation by providing evaluation of intermediate steps \cite{lightman2023let, wang2023mathshepherd, zhao2025genprm, luo2024omegaprm, zhang2024rest}. This fine-grained feedback improves training efficiency and final performance. However, obtaining high-quality process-level supervision signals requires high-cost manual annotation \cite{lightman2023let} or demands heavy computation to generate rollouts for estimation \cite{wang2023mathshepherd}.

Implicit PRM does not rely on explicit annotation \cite{yuan2024implicitprm, cui2025prime, guan2025rstar}. It is trained only with outcome signals and derives step rewards through reward parameterization without manual or automated annotations. Such implicit PRM is suitable for multi-step reasoning tasks such as MHQA because it can quantitatively evaluate the reasoning steps without explicit annotations. However, existing implicit PRM only focuses on the evaluation of intermediate steps in plain text scenarios (CoT), lacking consideration of the graph structure constraint of structured knowledge (KG). This overlook is critical because KGs provide explicit relational constraints to enhance LLM reasoning.
Further, designing separate implicit PRMs for CoT and KG and applying them independently in reasoning may cause semantic and logical inconsistency between top-scoring paths. This inconsistency may lead to conflict when integrating different paths to get the final answer.

To solve these problems, we propose DPRM, a dual implicit process reward model for automatic annotation and collaborative reasoning. It trains two PRMs, CoT-PRM and KG-PRM, and both of them derive process rewards from outcome signals via reward parameterization without explicit annotations. Among them, KG-PRM uses preference pairs to learn structural constraints from KGs. This mechanism compares reasoning paths in pairs, ranking them based on whether they follow graph rules and maintain logical consistency. In training, CoT preference pairs transformed from KG data enable CoT-PRM to learn logic structure, and KG preference pairs derived from CoT data introduce relational paths absent in the original KGs. In reasoning, both PRMs act as evaluators to screen KG paths and CoTs. Based on their rewards, high-quality KG triples and CoT steps guide and reinforce each other. This iterative interaction mechanism effectively fuses the advantages of bimodal information to make the model generate more accurate answers. In summary, our contributions are:

\begin{itemize}
    \item We propose a framework for dual-modality PRMs in MHQA. It automatically derives CoT-PRM and KG-PRM from ORMs via mathematical derivation without manual step-level annotations. The KG-PRM learns to capture topological constraints via path preference pairs.
    \item We design a bidirectional KG-CoT enhancement mechanism. In training, KG-PRM and CoT-PRM enhance each other through co-training. In reasoning, the two PRMs iteratively guide KG path acquisition and CoT generation to form a self-correcting reasoning. 
    \item Our method achieves the best on multiple datasets with up to $16.6\%$ improvement on Hit@1.

\end{itemize}

\section{Related Works}
\subsection{KG-enhanced LLM Reasoning}
Most of the existing MHQA methods based on LLMs and KGs follow the paradigm of GraphRAG.
GNN-RAG \cite{mavromatis2024gnn} first retrieves candidate answers on the KG and then uses Graph Neural Networks (GNNs) to extract the reasoning paths. G-Retriver \cite{he2025g} first obtains relevant nodes and edges and then finds the most relevant KG subgraph through the Rewarded Steiner Tree Algorithm. SubgraphRAG \cite{li2025simpleeffectiverolesgraphs} generates answers by retrieving relevant subgraphs. GRAG \cite{hu2025grag} introduces k-hop self-graph matching to obtain the most relevant KG subgraph.
Retrieval is usually completed before LLM reasoning with fixed content. Once the retrieval results are incomplete or redundant, LLM cannot dynamically update knowledge to supplement and correct the reasoning process.

Other works make the LLMs interact with the KG to generate answers step by step.  
KG-RAG \citep{sanmartin2024kg} searches the KG following the exploration plan from the LLM. GraphReader \citep{li2024graphreader} uses an LLM to explore the KG, continuously updating a notebook to record relevant information. Graph-CoT \citep{jin2024graph} makes LLMs traverse the KG step by step to find the key information. RoG \cite{luo2024rog} proposes a planning-search-reasoning framework.
In interactive systems, the planning and reflection parts depend on the LLM's evaluation, but this evaluation is qualitative rather than providing quantitative scores. Additionally, KG retrievers optimize for the shortest query-answer paths but lack path quality evaluation. If paths have errors, they will lead to wrong answers.

\subsection{Reward Model}
Early RMs mostly use explicit rewards. The ORMs assign rewards to the entire response and provide feedback after generating the final answer.
\citeauthor{christiano2017deep} first proposes an RM framework based on human preferences, laying the foundation for ORM. \citeauthor{ouyang2022training} proposes a three-stage RLHF framework for ORMs. Skywork-reward \cite{liu2024skywork} trains the RM with high-quality small samples through efficient preference data screening.
These methods can effectively replace manual reward design and make the model learn human preferences. However, ORMs fail to evaluate the reasoning process for complex tasks requiring multi-step reasoning. In addition, since rewards are emitted only at the terminal state, such sparse supervision signals struggle with stability and efficiency in training. 

\begin{figure*}[htbp] 
    \centering 
    \includegraphics[width=0.95\textwidth]{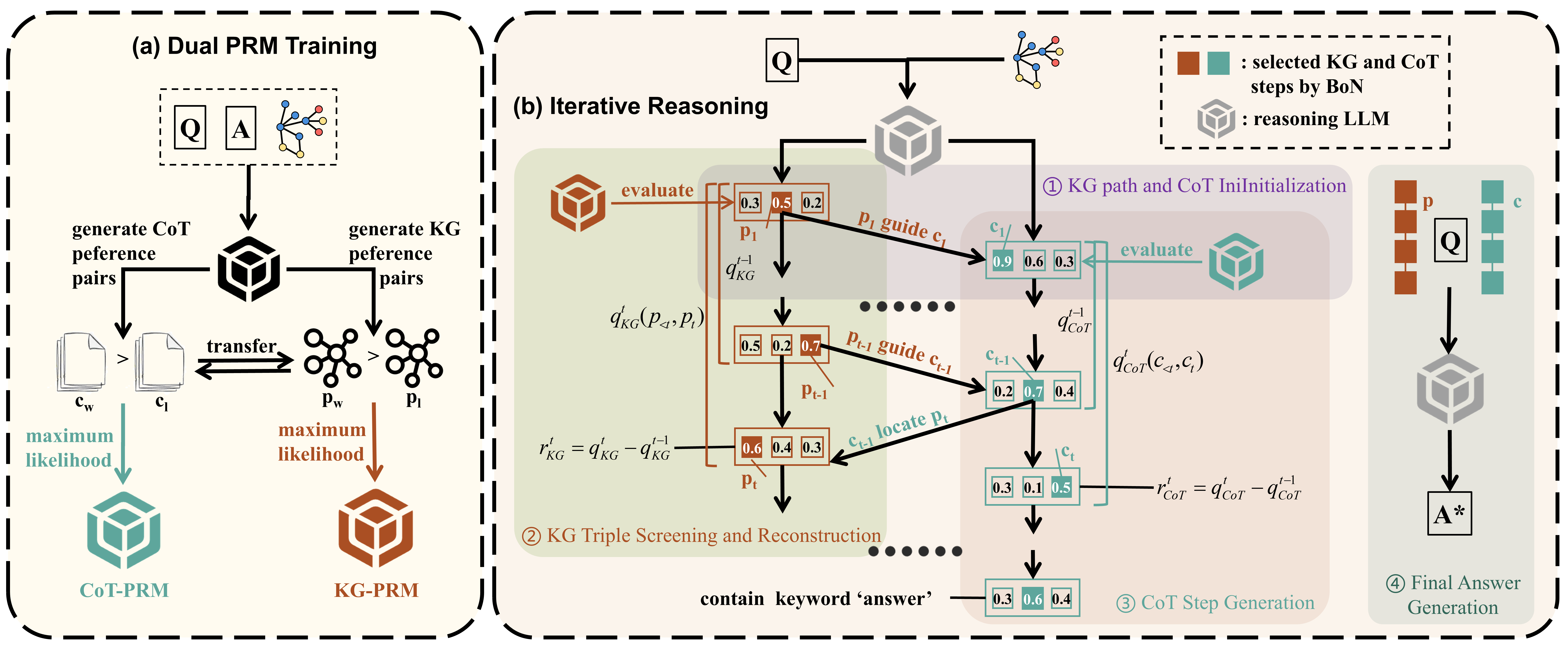} 
    \caption{The overview of DPRM. (a) \textbf{Dual PRM Training} trains CoT-PRM and KG-PRM with outcome signals. KG-PRM uses preference pairs to learn structural constraints. Co-training of both PRMs makes them mutually verify and collaboratively optimize the KG paths and CoTs. (b) \textbf{Iterative Reasoning} contains 4 parts: \ding{172} KG Path and CoT Initialization, \ding{173} KG triple screening and reconstruction (use KG-PRM), \ding{174} CoT step generation (use CoT-PRM), and \ding{175} Final Answer Generation. } 
    \label{fig:overview} 
\end{figure*}

PRMs break through this limitation by providing evaluations of intermediate steps. 
\citeauthor{lightman2023let} demonstrates that PRM solves the problem of incorrect step localization invalidation caused by sparse rewards. However, this method relies on humans to annotate reasoning steps, which is costly and difficult to generalize. Subsequent works adopt the methods of automatic annotation. 
Math-Shepherd \cite{wang2023mathshepherd} first proposes an automatic annotation framework for Monte Carlo Tree Search (MCTS). It uses LLM to generate the subsequent rollouts and calculates the probability of each step leading to the correct answer. GenPRM \cite{zhao2025genprm} adds code validation to this framework to filter logic error samples. OmegaPRM \cite{luo2024omegaprm} adds binary search to locate the first error step quickly. ReST-MCTS* \cite{zhang2024rest} uses MCTS to obtain positive samples to train PRM and then uses the trained PRM to guide the MCTS search.
This fine-grained feedback improves both training efficiency and final performance. However, obtaining high-quality process-level supervision signals requires high-cost manual annotation or demands heavy computation to generate rollouts for estimation.

Implicit PRM does not rely on explicit annotation. Implicit rewards originate from Direct Preference Optimization (DPO) \cite{rafailov2023dpo}. It transforms the reinforcement learning (RL) objective into a binary cross-entropy loss. Subsequent works explore how to derive PRMs through implicit rewards. ImplicitPRM \cite{yuan2024implicitprm} is trained only with outcome signals and derives step rewards from full response rewards through reward parameterization. PRIME \cite{cui2025prime} combines implicit PRM with RL to compute step rewards and avoids reward hacking through an online mechanism. rStar-Math \cite{guan2025rstar} proposes a self-evolving framework based on implicit PRM.
These methods can effectively reduce the annotation cost of rewards. However, existing implicit PRM only focuses on the evaluation of intermediate steps in plain text scenarios (CoT), lacking consideration of the graph structure constraint of structured knowledge (KG). Further, it cannot cope with the collaborative reasoning process of KGs and CoTs.

\section{Method}
\subsection{Dual PRM Architecture}
Given a question $Q$ and a related KG $G$, we aim to obtain an accurate answer $A*$. First, we train two implicit PRMs with outcome supervision to obtain evaluators that can quantify the quality of KG paths and CoTs, called KG-PRM and CoT-PRM. Then, in reasoning, we use the two PRMs to evaluate the KG path steps and CoT steps to screen the better ones. The selected KG path steps will guide the generation of the CoT steps, and the entities in the CoT steps will help to capture the next KG path steps. They complement each other and interact iteratively to get the accurate final answers. Our method contains two key parts: \textbf{Dual PRM Training} and \textbf{Iterative Reasoning}. The overall framework is shown in Figure \ref{fig:overview}. The prompt templates are in Appendix E.

\subsubsection{Dual PRM Training} 
In MHQA tasks, since the quality of CoT and KG reasoning is highly correlated, we adopt a data interaction strategy to make CoT-PRM and KG-PRM learn from two types of preference pairs. These PRMs use special reward functions in training so they can be trained with outcome signals and used to evaluate the reasoning process, which is theoretically demonstrated.
Specifically, we first initialize CoT-PRM and KG-PRM with the data of the corresponding modality. After that, we transform KG data into CoTs to train CoT-PRM, and transform CoT data into KG paths to train KG-PRM. KG-PRM uses preference pairs to learn structural constraints from KGs. 

\subsubsection{Iterative Reasoning} After training, we use the PRMs to evaluate KG paths and CoTs. 
In each iteration, the LLM retrieves and reconstructs new KG triples based on the question $Q$ and the previous CoT $\{c_1,...,c_{i-1}\}$. The KG-PRM evaluates the $N$ triples and selects the best as $p_i$. $p_i$ will then guide the generation of the next CoT steps. The CoT-PRM evaluates the samples to select the best as $c_i$. The iteration continues until $c_i$ reaches the final answers.

\subsection{Dual PRM Training Framework}
\subsubsection{PRM Initialization}
We first use the KG and CoT data to initialize the corresponding KG-PRM and CoT-PRM. The goal of initialization is to establish foundational evaluation criteria for the two PRMs. We directly optimize the model's preference probability distribution over preference pairs. KG-PRM learns to capture topological constraints for long chain reasoning via path preference pairs.

For KG-PRM, each piece of training data includes the question $Q$, true KG path $p_w$ ($w$ stands for win), and false KG path $p_l$ ($l$ stands for lose), which is $(Q, p_w, p_l)$. 
KG path is an alternate sequence of entities and relations, starting with the question entity and ending with the answer entity. 
\begin{equation}
    \left\{ e_Q \xrightarrow{r_1} e_1 \xrightarrow{r_2} \ldots \xrightarrow{r_n} e_A \right\}
\end{equation}
where $e_i$ are entities in intermediate steps. Each step in the KG path is a triple like $(e_Q, r_1, e_1)$. Given the paths associated with the graph structure, the KG-PRM can learn to capture topological constraints. The true samples in the data are KG paths from the question entities to the ground truth entities. And as shown in Figure \ref{fig:tf}, the false samples are obtained in the following ways: (1) Factual errors (The entities in triples are wrong.); (2) Logical errors (The relations in triples are wrong.); (3) Logical break (The entities in adjacent triples of a KG path cannot be aligned or connected.).

\begin{figure}[!htbp] 
    \centering 
    \includegraphics[width=0.45\textwidth]{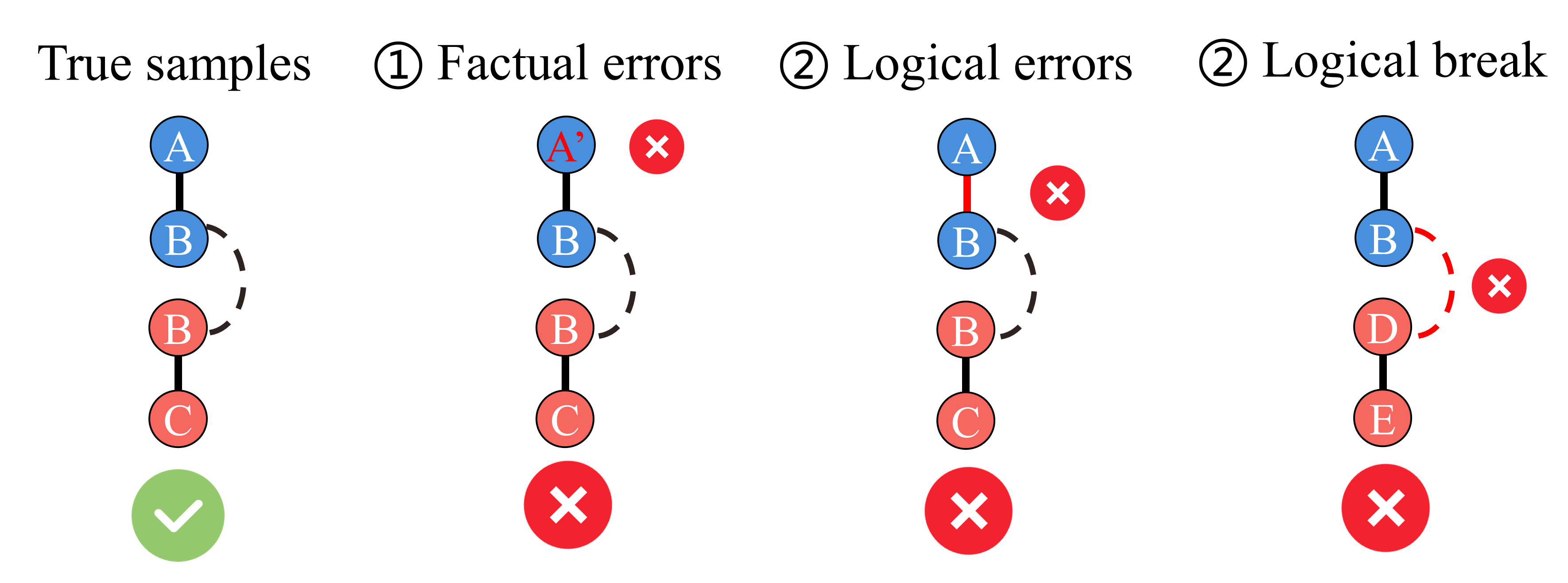} 
    \caption{True samples and false samples.} 
    \label{fig:tf} 
\end{figure}

For a question $Q$, a KG path $p$, a KG policy model $\pi_{\phi}^{KG}$, and a reference model $\pi_{ref}$, the reward function is:
\begin{equation}
r_{KG} (p \mid Q)=\gamma \log \frac{\pi_{\phi}^{KG}(p \mid Q)}{\pi_{ref}(p \mid Q)}
\end{equation}
where $\gamma$ is a hyperparameter used to control the strength of the preference signal. In training, the policy model is trained to be consistent with human preferences, and the reference model guarantees stability during optimization. Specifically, the loss function $L_{KG}$ is to maximize the log probability of the reward score difference between true and false samples. For data $(Q, p_w, p_l)$, the loss function is:
\begin{multline}
 L_{KG}=  -\log \sigma \left(\gamma \log \frac{\pi_{\phi}^{KG}\left( p_w \mid Q\right)}{\pi_{ref} \left(p_w \mid Q\right)} \right.\\
 \left. -\gamma \log \frac{\pi_{\phi}^{KG}\left(p_l \mid Q\right)}{\pi_{ref} \left(p_{l} \mid Q\right)} \right)
\end{multline}
where $\sigma$ is the sigmoid function. Then KG-PRM can learn to identify the accurate KG paths that follow KG logic rules. 

Even though KG-PRM is trained with only outcome signals, it can be used to evaluate the KG path steps because the following proposition:

\begin{proposition}\label{prop:key} (Proof in Appendix A)
Consider an ORM where the reward is  
$r_{\theta} (y)=\beta \log \frac{\pi_{\theta}(y)}{\pi_{ref}(y)}$.
Define
$q_{\theta}^{t}\left(y_{<t}, y_{t}\right):=\sum_{i=1}^{t} \beta \log \frac{\pi_{\theta}\left(y_{i} \mid y_{<i}\right)}{\pi_{r e f}\left(y_{i} \mid y_{<i}\right)}$.
$q_\theta^t$ is the exponential average of $r_\theta$ at step $t$.
\[
q_{\theta}^{t}\left(y_{<t}, y_{t}\right)=\beta \log \mathbb{E}_{\pi_{ref}\left(y\mid y_{\leq t}\right)} e^{\frac{1}{\beta} r_{\theta}{\left(y\right)}}
\]
Hence, $q_\theta^t$ represents an exact expectation of outcome reward $r_\theta$ at step $t$. 
\end{proposition}

So the KG-PRM can learn the expectation of outcome reward $r_{KG}$ at step $t$ to calculate the process reward. Following the previous works \cite{lu2024autopsv}, we define the process reward at step $t$ as the difference between $q^{t}$ and $q^{t-1}$. The process reward at step $t$ for KG-PRM is,
\begin{equation}
r_{KG}^{t}  =q_{KG}^{t}-q_{KG}^{t-1}=\sum_{i=t-1}^{t} \gamma \log \frac{\pi_{\phi}^{KG}\left(p_{i} \mid p_{<i}\right)}{\pi_{ref}\left(p_{i} \mid p_{<i}\right)}
\end{equation}
where $p_i$ is the triple at step $i$, and $p_{<i}$ is the KG path of the previous $i$ steps. This reward is used to evaluate the quality of generated KG path steps in the iterative reasoning process (KG Triple Screening and Reconstruction in Section \ref{IRF}). 

For CoT-PRM, each training data includes three parts: the question $Q$, true CoT $c_w$, and false CoT $c_l$, which is $(Q, c_w, c_l)$. Among them, the true samples obtain the ground truth, and the false samples are obtained in the following ways: (1) Factual errors (wrong content in reasoning steps); (2) Step skipping (omitting some reasoning steps). (3) Step redundancy (Some steps are unrelated to the QA pairs.).
Similar to KG-PRM, we adopt the causal LM log-likelihood ratio as the reward function. For a CoT $c$ and a CoT policy model $\pi_{\theta}^{CoT}$, the reward function is:
\begin{equation}
r_{CoT} (c \mid Q)=\beta \log \frac{\pi_{\theta}^{CoT}(c \mid Q)}{\pi_{ref}(c \mid Q)}
\end{equation}
where $\beta$ is used to control the signal strength. For data $(Q, c_w, c_l)$, the loss function of CoT-PRM is:
\begin{multline}
L_{CoT}=-\log \sigma\left(\beta \log \frac{\pi_{\theta}^{CoT}\left(c_{w} \mid Q\right)}{\pi_{ref}\left(c_{w} \mid Q\right)} \right.\\
 \left. -\beta \log \frac{\pi_{\theta}^{CoT}\left(c_{l} \mid Q\right)}{\pi_{ref}\left(c_{l} \mid Q\right)}\right)
\end{multline}

Then CoT-PRM can distinguish more complete and concise chains and identify factual errors. Similarly, the process rewards at step $t$ for CoT-PRM is,
\begin{equation}
    r_{CoT}^{t} =q_{CoT}^{t}-q_{CoT}^{t-1}=\sum_{i=t-1}^{t} \beta \log \frac{\pi_{\theta}^{CoT}\left(c_{i} \mid c_{<i}\right)}{\pi_{ref}\left(c_{i} \mid c_{<i}\right)}
\end{equation}
where $c_i$ is the CoT step at step $i$, and $c_{<i}$ is the previous $i$ CoT steps. In this way, CoT-PRM can be trained with outcome signals and used to evaluate the CoT steps in the iterative reasoning (CoT Step Generation in Section \ref{IRF}).

\subsubsection{Alternating Optimization}
After initialization, the two implicit PRMs learn to distinguish between true and false samples. Since the quality of CoT and KG paths is highly correlated, we design a co-training strategy, enabling both implicit PRMs to learn from dual data modalities.

The core of co-training is to enhance the implicit PRMs with data from both modalities. For data $(Q, p_w, p_l)$ from the KG, we convert the true sample $p_w$ and the false sample $p_l$ into continuous texts as the corresponding preference pairs $cp_w$ and $cp_l$. Then the derived data $(Q,cp_w,cp_l)$ is used to train the CoT-PRM. The loss function $L_{CoT-p}$ still maximizes the probability of the true sample $cp_w$ and minimizes the probability of the false sample $cp_l$, that is,
\begin{equation}
L_{CoT-p}=-\log \sigma\left(r_{CoT}\left(cp_{w}\right)-r_{CoT}\left(cp_{l}\right)\right)
\end{equation}

For the data from CoT $(Q, c_w, c_l)$, we extract triples in each step and connect them into entity-relation paths to get KG paths $pc_w$ and $pc_l$. After that, we construct preference signals of KG paths $pc_w$ and $pc_l$ to train the KG-PRM.
\begin{equation}
L_{KG-c}=-\log \sigma\left(r_{KG}\left(pc_{w}\right)-r_{KG}\left(pc_{l}\right)\right)
\end{equation}

Through co-training, the KG-derived CoT preference pairs with structure constraints enable the CoT-PRM to capture topological patterns for long-chain reasoning. Meanwhile, CoT-derived KG preference pairs expand training diversity for KG-PRM by introducing relational paths absent in the original KG. These two PRMs will subsequently serve as evaluators for evaluating KG paths and CoTs.

\subsection{Iterative Reasoning Framework}
\label{IRF}

The reasoning process can be divided into four parts: \textbf{KG Path and CoT Initialization}, \textbf{KG triple screening and reconstruction}, \textbf{CoT step generation}, and \textbf{Final Answer Generation}. The process is shown in Algorithm \ref{alg:kg_cot}. 
We apply KG-PRM and CoT-PRM in the reasoning process to evaluate the KG paths and CoTs. Then they are selected through the Best-of-N (BoN) method based on the process reward.

\subsubsection{KG Path and CoT Initialization}
Before the iteration, we first get the initial KG triple and the first CoT step. Specifically, all-mpnet-base-v2 \cite{reimers2019sbert, song2020mpnet} obtains the embeddings of question $Q$ and all KG triples $tr\in G$. The set $tr_1$ is composed of the top $m$ similar triples to $Q$ according to the cosine similarity. After that, the LLM retrieves the KG triple $p_1$ starting from the entities $e_1$ in question $Q$ from the set $tr_1$. The selection is based on the rewards from KG-PRM through the BoN method.
Then it will be submitted to the LLM along with $q$ to generate the corresponding CoT step $c_1$. $c_1$ is also screened by CoT-PRM through the BoN method.
At this point, the CoT $c$ and the KG path $p$ are initialized with $c=\{c_1\}$ and $p=\{p_1\}$. After that, it goes into the iterative process.

\begin{algorithm}[t]
\caption{Reasoning Process}\label{alg:kg_cot}
\begin{algorithmic}[1]
\REQUIRE question $Q$, knowledge graph $G$, max iterations $n$
\ENSURE Final answer $A^*$

\STATE $tr_1 \gets \operatorname{top-m}\left(\operatorname{sim}\left(Q, tr\right)\right), tr \in G$ \COMMENT{$tr$: KG triple.}
\STATE $p_1 \gets \operatorname{LLM\_retrieve}(Q, tr_1)$ \COMMENT{$p_1$: Initial KG triple.}
\STATE $c_1 \gets \operatorname{LLM\_generate}(Q, p_1)$ \COMMENT{$c_1$: First CoT step.}
\STATE $p \gets \{p_1\},\; c \gets \{c_1\},\; i \gets 2$ \COMMENT{Initialization.}
\WHILE{$i \leq n+1$}
    \STATE $s_i \gets Q \oplus c_{i-1}$ \COMMENT{$s_i$: concatenation for selecting.}
    \STATE $tr_i \gets \operatorname{top-m}\left(\operatorname{sim}\left(s_i, tr\right)\right), tr \in G$
    \STATE $p_i \gets \operatorname{LLM\_retrieve\_reconstruct}(Q, c_{i-1}, tr_i)$ \COMMENT{Get KG triple.}
    \STATE $p \gets p \cup \{p_i\}$ \COMMENT{Update KG path.}
    \STATE $c_i \gets \operatorname{LLM\_generate}(Q, p_i, c)$ \COMMENT{Get CoT step.}
    \STATE $c \gets c \cup \{c_i\}$ \COMMENT{Update CoT.}
    
    \IF{``answer'' \textbf{ in } $c_i$} 
        \STATE \textbf{break} \COMMENT{``answer'' is a keyword.}
    \ENDIF
    
    \STATE $i \gets i + 1$
\ENDWHILE

\STATE $\mathrm{prompt}_{hard} \gets [Q, c, p]$ \COMMENT{Plain text prompt.}
\STATE $\mathrm{prompt}_{soft} \gets \operatorname{text}(p)$ \COMMENT{Graph structure prompt.}
\STATE $A^* \gets \operatorname{LLM\_generate}(\mathrm{prompt}_{hard} , \mathrm{prompt}_{soft})$
\STATE \textbf{return} $A^*$
\end{algorithmic}
\end{algorithm}

\begin{table*}[htbp]  
  \centering  
  \begin{tabular}{cccccc}  
    \toprule
    \multirow{2}{*}{Category} & \multirow{2}{*}{Method} & \multicolumn{2}{c}{WebQSP} & \multicolumn{2}{c}{CWQ}\\
    \cmidrule(lr){3-4} \cmidrule(lr){5-6}
    & & F1 Score & Hit@1 & F1 Score & Hit@1 \\
    \midrule
    \multirow{3}{*}{LLM only} & Qwen2-7B \cite{yang2024qwen2technicalreport} & 0.3550 & 0.5080 & 0.2160 & 0.2530 \\ 
    & Llama-2-7B \cite{touvron2023llama} & 0.3650 & 0.5640 & 0.2140 & 0.2840 \\ 
    & Llama-3.1-8B \cite{dubey2024llama} & 0.3480 & 0.5550 & 0.2240 & 0.2810 \\ 
    \midrule
    \multirow{5}{*}{KG+LLM} & G-Retriever \cite{he2025g} & 0.4674 & 0.6808 & 0.3396 & 0.4721 \\ 
    & GRAG \cite{hu2025grag} & 0.5022 & 0.7236 & 0.3649 & 0.5018 \\ 
    & SubgraphRAG \cite{li2025simpleeffectiverolesgraphs} & 0.7057 & \underline{0.8661} & 0.4716 & 0.5698 \\
    & RoG \cite{luo2024rog} & 0.7080 & 0.8570 & 0.5620 & \underline{0.6260} \\
    & GNN-RAG \cite{mavromatis2024gnn} & \underline{0.7130} & 0.8060 & \underline{0.5940} & 0.6170 \\
    \midrule
    \multirow{2}{*}{ORM+LLM} & SkyworkRM-Llama3.1-8B \cite{liu2024skywork} & 0.5751 & 0.7329 & 0.3465 & 0.5755\\ 
    & ArmoRM-Llama3-8B \cite{wang2024armorm} & 0.5816 & 0.7198 & 0.3345 & 0.5531 \\ 
    \midrule
    \multirow{3}{*}{PRM+LLM} & RLHFlow-8B-Mistral-Data \cite{xiong2024rlhf1} & 0.5863 & 0.7320  & 0.3431 & 0.5540 \\ 
    & RLHFlow-8B-DeepSeek-Data \cite{dong2024rlhf2} & 0.5356 & 0.6815 & 0.3572 & 0.5718 \\ 
    & ImplicitPRM-DPO \cite{yuan2024implicitprm} & 0.5744 & 0.7266 & 0.3538 & 0.5650 \\ 
    \midrule
    \multirow{2}{*}{Our method} & DPRM+Llama-3.1-8B & \textbf{0.7228} & 0.8715 & \textbf{0.5988} & \textbf{0.7301} \\
    & DPRM+Llama-2-7B & 0.7220 & \textbf{0.8718} & 0.5886 & 0.7299 \\
    \bottomrule
  \end{tabular}
  \caption{ Model performance on two datasets comparing five categories of methods. The best results are \textbf{bolded}, and the best results in baselines are \underline{underlined}.}  
  \label{tab:main}  
\end{table*}

\subsubsection{KG Triple Screening and Reconstruction}
To obtain the KG path at step $i$, we fetch the next KG triple based on the question $Q$ and the previous CoT step $c_{i-1}$. First, we concatenate them together to get the string $s_{i}=Q \oplus c_{i-1}$.
After that, $s_i$ will be input into the all-mpnet-base-v2 to obtain the embedding, and select the set $tr_i$ composed of the top $m$ similar triples to it. The LLM retrieves the most relevant triples from $tr_i$ according to the question $Q$ and the previous CoT step $c_{i-1}$ and reconstructs them. Reconstruction refers to the structural adjustment of the selected triples to ensure that the head entities of the triples come from the entities $e_{i-1}$ in the CoT step $c_{i-1}$, as shown in Figure \ref{fig:re}. 

\begin{figure}[!htbp] 
    \centering 
    \includegraphics[width=0.2\textwidth]{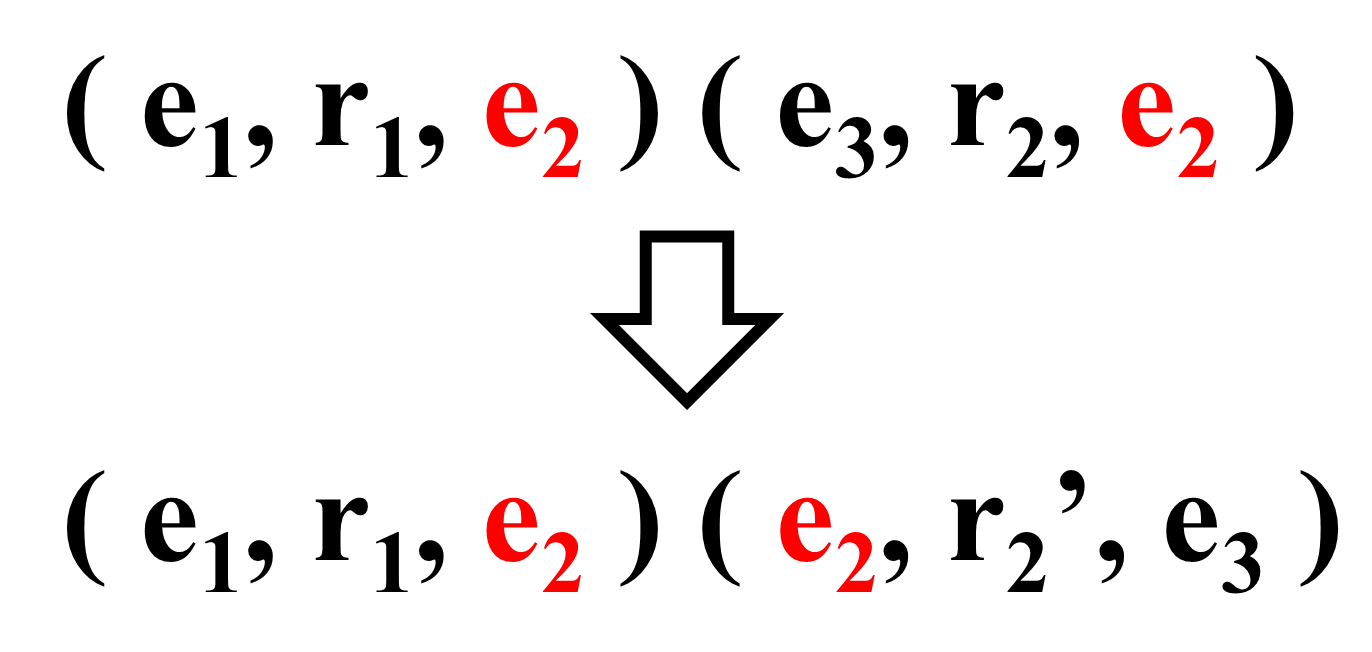} 
    \caption{Triple reconstruction for entity consistency.} 
    \label{fig:re} 
\end{figure}

This process obtains the next KG triple $p_i$ based on the rewards from KG-PRM through the BoN method and updates the current KG path to $p=\{p_1, \ldots\,p_i\}$. This triple $p_i$ will guide the expansion of CoT in the next part.

\subsubsection{CoT Step Generation}
We generate the next CoT step based on the KG triple $p_i$, and the entities in it will guide the selection of the next KG triple. We concatenate the question $Q$, the KG triple $p_i$ and the previous CoT $c=\{c_1, \ldots\,c_{i-1}\}$ according to the prompt template. Then we submit the prompt to make the LLM generate the next CoT step. This process obtains the next CoT step $c_i$ based on the rewards from CoT-PRM through the BoN method and updates the current CoT to $c=\{c_1, \ldots\,c_i\}$.

The iteration stops when CoT $c$ contains the keyword "answer" or reaches the preset limit. The CoT $c$ and the KG path $p$ will support the final answer generation.

\subsubsection{Final Answer Generation}
In generation, the prompts submitted to the LLM are divided into hard prompts (plain text) and soft prompts (textualized graph structures) \cite{hu2025grag,he2025g}. They are submitted to the LLM to generate the final answer $A*$. By combining them, text and graph information can both be preserved to reduce knowledge loss. The prompt templates are in Appendix E. 

\section{Experiments}

\subsection{Experiment Setup}
\subsubsection{Datasets.} Following previous works \citep{luo2024rog, wang2023knowledge, luo2024graph}, we conduct experiments on two datasets, WebQuestionSP (WebQSP) \citep{yih2016value} and Complex WebQuestions (CWQ) \citep{talmor2018web}. Details of datasets are in Appendix C.

\subsubsection{Baselines.} We compare DPRM with 13 baselines grouped into 4 categories: 
(1) LLM only,
including Qwen2-7B \cite{yang2024qwen2technicalreport}, Llama-2-7B \cite{touvron2023llama}, and Llama-3.1-8B \cite{dubey2024llama}.
(2) KG+LLM,
including G-Retriever \cite{he2025g}, GRAG \cite{hu2025grag}, SubgraphRAG \cite{li2025simpleeffectiverolesgraphs}, RoG \cite{luo2024rog}, and GNN-RAG \cite{mavromatis2024gnn}.
(3) ORM+LLM,
including SkyworkRM-Llama3.1-8B \cite{liu2024skywork} and ArmoRM-Llama3-8B \cite{wang2024armorm}.
(4) PRM+LLM,
including RLHFlow-8B-Mistral-Data \cite{xiong2024rlhf1}, RLHFlow-8B-DeepSeek-Data \cite{dong2024rlhf2}, and ImplicitPRM-DPO \cite{yuan2024implicitprm}.
Details of baselines are in Appendix D.

\subsubsection{Evaluation Metrics.} Following previous works \cite{luo2024rog,hu2025grag,luo2024graph}, we use Hit@1 and the F1 score as evaluation metrics. Hit@1 checks if the ground truth exists in the generated answers. The F1 score is a harmonic average of accuracy and recall, providing a metric that balances false positives and false negatives.

\subsubsection{Implementations.} We train PRMs based on Llama-3.1-8B-Instruct with $\beta=0.05$ and $\gamma=0.05$, which are empirically determined. We select the number of samples $N=8$ for screening and the number of top-m related triples $m=25$. The maximum number $n$ of iterations is set to 4, considering that the data is at most 4 hops. More details for training and hyperparameter settings are in Appendix B.

\subsection{Main Results}
We compare our method, DPRM, to other baselines on the datasets. As Table \ref{tab:main} shows, DPRM performs best on both datasets. The F1 score on WebQSP and CWQ is $1.4\%$ and $0.8\%$ above the best baseline, and Hit@1 is $0.7\%$ and $16.6\%$ above. It shows that DPRM can effectively enhance the reasoning ability of the LLM. 

The overall performance of KG+LLM baselines is better than that of LLM only baselines, indicating that KGs are important in MHQA tasks. It is also found that RMs can effectively improve the reasoning ability of LLMs in MHQA tasks. The ORM+LLM baselines increase the F1 score of the corresponding LLM by up to $67.1\%$ and Hit@1 by up to $104.8\%$. Similarly, the PRM+LLM baselines increase the F1 score by up to $68.5\%$ and Hit@1 by up to $103.5\%$. However, since these traditional RMs fail to consider the structural constraints of KGs, their overall performance is worse than that of KG+LLM baselines. This again highlights the importance of KG. In addition, because MHQA tasks involve fewer reasoning steps than mathematical reasoning tasks, the performance gap between ORMs and PRMs is not very big. On the webqsp dataset where questions are 1-2 hops, the worst-performing RM is even a PRM (RLHFlow-8B-DS-Data). This indicates that fewer reasoning steps in MHQA tasks can limit the advantage of PRMs.

We also observe that larger LLMs do not always perform better than smaller LLMs. For the performance of Llama-2-7B and Llama-3.1-8B, Llama-2-7B has a higher F1 and Hit@1 on the WebQSP dataset, by $4.9\%$ and $1.6\%$. In the CWQ dataset, its Hit@1 is still higher. This suggests that increasing the parameters does not inherently enhance the graph reasoning ability of LLMs.

\subsection{Ablation Study}

\begin{table}[!htbp]  
  \centering  
  \setlength{\tabcolsep}{1mm}
  \small
  \begin{tabular}{ccccc}  
    \toprule
    \multirow{2}{*}{Method} & \multicolumn{2}{c}{WebQSP} & \multicolumn{2}{c}{CWQ}\\
    \cmidrule(lr){2-3} \cmidrule(lr){4-5}
    & F1 Score & Hit@1 & F1 Score & Hit@1 \\
    \midrule
    DPRM & \textbf{0.7228} & \textbf{0.8715} & \textbf{0.5988} & \textbf{0.7301} \\ 
    w/o Co-training & 0.7009 & 0.8581 & 0.5419 & 0.6824 \\
    w/o Iteration & 0.7095 & 0.8587 & 0.5824 & 0.7234 \\
    w/o both & 0.6842 & 0.8500 & 0.5266 & 0.6692 \\
    \bottomrule
  \end{tabular}
  \caption{Performances of three model variables.}  
  \label{tab:ablation}  
\end{table}

We conduct a series of evaluations of DPRM to see which component plays a key role, including removing co-training in training, removing iteration in reasoning, and removing both of the two parts. 
We can see the performances of variables all decrease, as shown in Table \ref{tab:ablation}. It indicates that every component is indispensable.
Among them, removing co-training drops model performance more than removing iteration. On the CWQ dataset, F1 of the former decreased by $6.8\%$, and Hit@1 decreased by $5.5\%$ more than that of the latter. The situation is similar on WebQSP. This suggests that co-training plays a more central role in reasoning. It enables the CoT-PRM to capture topological patterns for long-chain reasoning and expand training diversity for KG-PRM by introducing relational paths absent in the original KG. 

\subsection{Analytical Experiments}

\subsubsection{Topological Pattern Learning for CoT-PRM}

\begin{figure}[!htbp] 
    \centering 
    \includegraphics[width=0.3\textwidth]{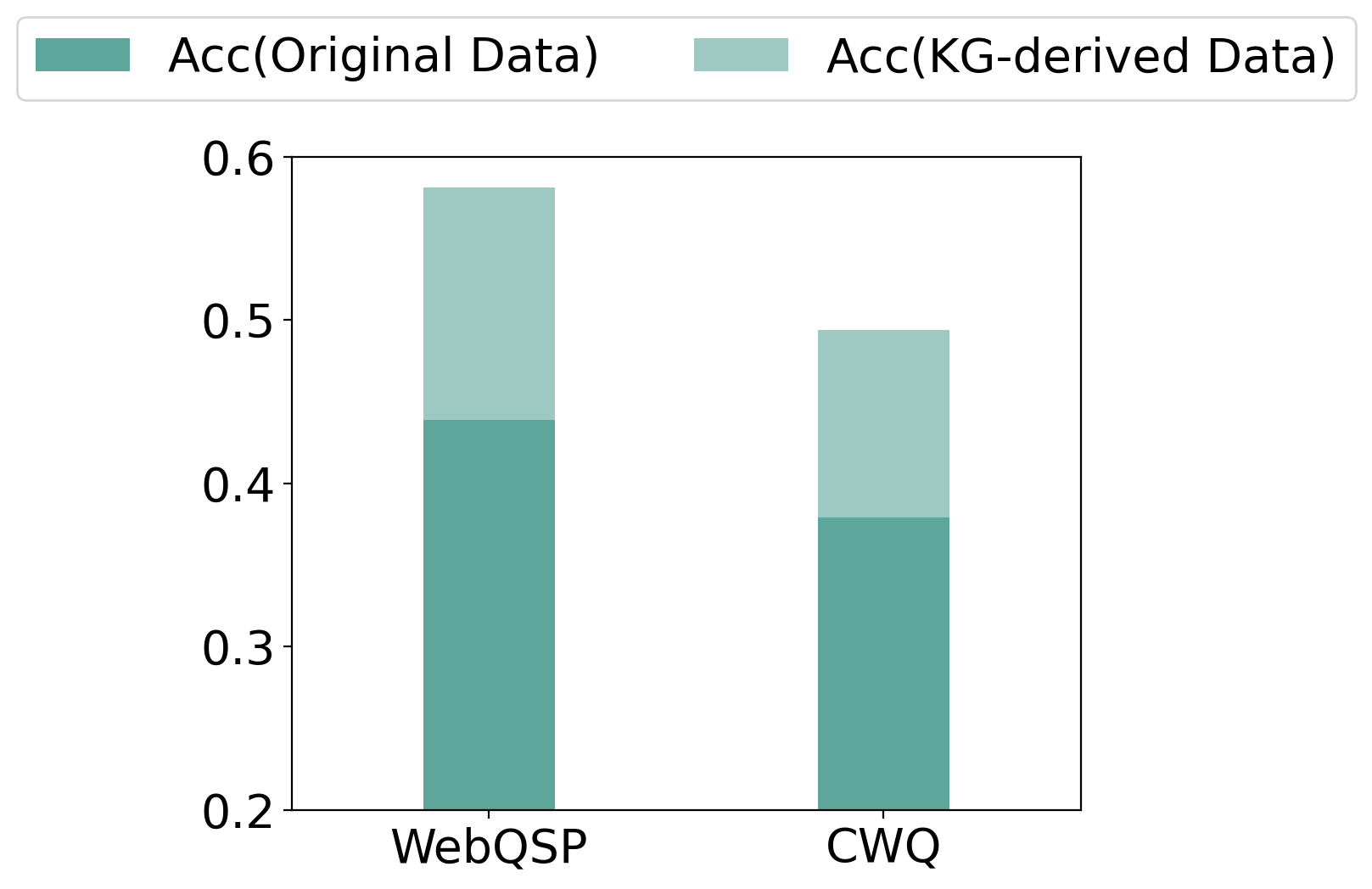} 
    \caption{Performances on the CoT-PRMs trained on original and KG-derived data.} 
    \label{fig:cocot} 
\end{figure}

To test whether the data derived from KG paths allows CoT-PRM to learn topological patterns, we randomly generate 800 new KG preference pairs and use the CoT-PRMs trained on original and KG-derived data to evaluate them. It tests whether the two CoT-PRMs can distinguish true and false samples of structured data. As shown in Figure \ref{fig:cocot}, the accuracy of the CoT-PRM trained on KG-derived data is higher than that of the original-data-trained CoT-PRM on both datasets. It indicates that co-training enables the CoT-PRM to capture topological patterns for long-chain reasoning. 

\subsubsection{New Relational Paths for KG-PRM}

\begin{figure}[!htbp] 
    \centering 
    \includegraphics[width=0.45\textwidth]{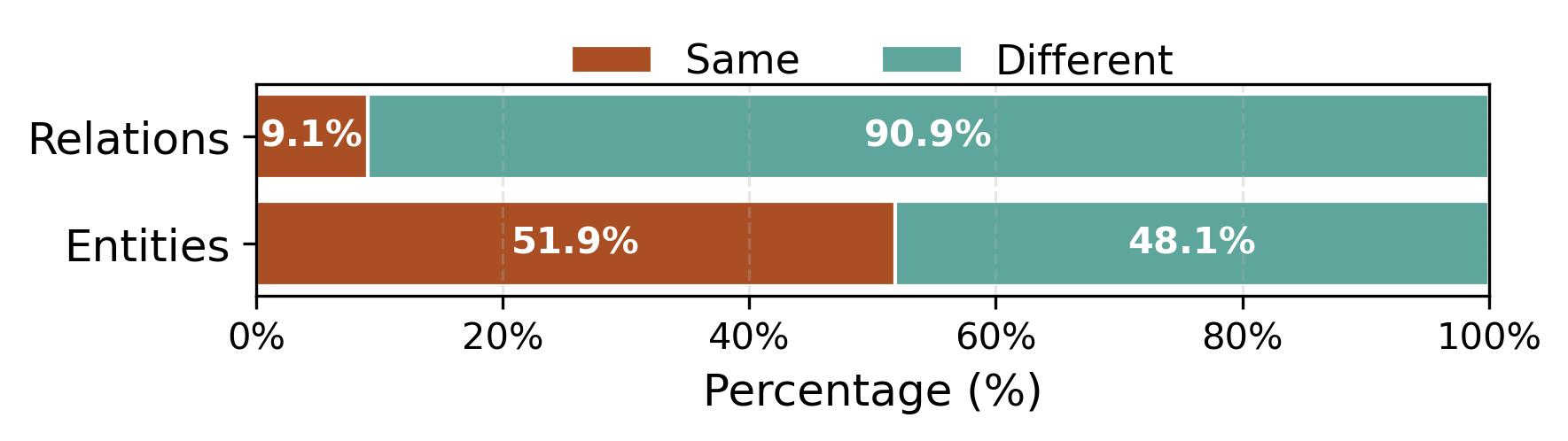} 
    \caption{Proportion of Same and Different Entities/Relations in CoT-Derived KG Data.} 
    \label{fig:cokg} 
\end{figure}

To evaluate the impact of co-training on the training data for KG-PRM, we calculate the proportion of entities and relations in the CoT-derived KG that are the same as and different from those in the original KG. As shown in Figure \ref{fig:cokg}, it keeps some key entities but shows big differences in relations. Specifically, the CoT-derived KG data does not follow standard KG relation formats like ``celebrities.romantic\_relationship.celebrity'', but instead uses more accurate terms like ``marry''.
It indicates that CoT-derived KG data expand training diversity for KG-PRM by introducing paths absent in the original KG.

\subsubsection{KG Structural Significance}

\begin{table}[htbp]  
  \centering  
  \setlength{\tabcolsep}{1mm}
  \small
  \begin{tabular}{ccccc}  
    \toprule
    \multirow{2}{*}{Method} & \multicolumn{2}{c}{WebQSP} & \multicolumn{2}{c}{CWQ}\\
    \cmidrule(lr){2-3} \cmidrule(lr){4-5}
    & F1 Score & Hit@1 & F1 Score & Hit@1 \\
    \midrule
    SkyworkRM & 0.5751 & 0.7329 & 0.3465 & 0.5755 \\
    ArmoRM & 0.5816 & 0.7198 & 0.3345 & 0.5531 \\
    RLHFlow-Mistral & 0.5863 & 0.7320  & 0.3431 & 0.5540 \\
    RLHFlow-DeepSeek & 0.5356 & 0.6815 & 0.3572 & 0.5718 \\
    ImplicitPRM-DPO & 0.5744 & 0.7266 & 0.3538 & 0.5650 \\
    KG-PRM (Our method) & \textbf{0.5946} & \textbf{0.7346} & \textbf{0.4283} & \textbf{0.6341} \\
    \bottomrule
  \end{tabular}
  \caption{Performances of RMs.}  
  \label{tab:iprm}  
\end{table}

As shown in Table \ref{tab:iprm}, our KG-PRM outperforms other RMs, especially on CWQ, which has more multi-hop questions. This shows that our method can make KG-PRM consider the graph structure constraint in KGs and learn to identify accurate KG paths that follow KG logic rules. In addition, in MHQA tasks, structural knowledge is essential, which can significantly enhance the LLM reasoning.

\subsubsection{Partial Data Robustness}

\begin{table}[htbp]  
  \centering  
  \setlength{\tabcolsep}{1mm}
  \small
  \begin{tabular}{ccccc}  
    \toprule
    \multirow{2}{*}{Method} & \multicolumn{2}{c}{WebQSP} & \multicolumn{2}{c}{CWQ}\\
    \cmidrule(lr){2-3} \cmidrule(lr){4-5}
    & F1 Score & Hit@1 & F1 Score & Hit@1 \\
    \midrule
    DPRM (full data) & 0.7228 & 0.8715 & 0.5988 & 0.7301 \\ 
    DPRM (partial data) & 0.6670 & 0.8089 & 0.5843 & 0.7190 \\
    \bottomrule
  \end{tabular}
  \caption{Performances with Full/Partial Training Data.}  
  \label{tab:partial}  
\end{table}

We evaluate DPRM with CoT-PRM and KG-PRM trained on partial datasets. Each PRM is fine-tuned on 1,000 preference pairs of CoTs or KG paths, respectively. As shown in Table \ref{tab:partial}, DPRM achieves $92.8\%$ of the full-data performance using only one-eighth of the training data. Compared with the results in Table \ref{tab:main}, our DPRM trained on partial data still outperforms 10 baselines.

\section{Conclusion}
In this paper, we propose a dual implicit process reward model (DPRM) for automatic annotation and collaborative reasoning in MHQA tasks. It trains KG-PRM and CoT-PRM to assign rewards for CoT and KG path steps during reasoning without using step-wise annotations. Different from existing implicit PRMs, KG-PRM is designed for capturing the structural information of KG; co-training of both PRMs enhances the consistency between CoT and KG: CoT-PRM learns to capture the topological patterns in KG, and KG-PRM learns from more relational and diverse paths. The experiment results support the above statements and show that DPRM achieves state-of-the-art (SOTA) performance among 13 baselines. 


\section*{Acknowledgments}
This paper is partially supported by National Natural Science
Foundation of China (NSFC Grant No. 62576353 and No. 62306330) and National Natural Science Foundation of China Distinguished Young Scholar Project (Grant No. 62125604).

\appendix

\bibliography{aaai2026}
\onecolumn
\renewcommand{\thesection}{\Alph{section}}
\frenchspacing

\setcounter{secnumdepth}{2}

\title{Appendix}

\maketitle

\section{Proof of Proposition}

\begin{proposition}
Consider an ORM where the reward is parameterized by the log-likelihood ratio of two causal LMs, i.e. 
$r_{\theta} (y)=\beta \log \frac{\pi_{\theta}(y)}{\pi_{ref}(y)}$.
Define
$q_{\theta}^{t}\left(y_{<t}, y_{t}\right):=\sum_{i=1}^{t} \beta \log \frac{\pi_{\theta}\left(y_{i} \mid y_{<i}\right)}{\pi_{r e f}\left(y_{i} \mid y_{<i}\right)}$.
$q_\theta^t$ is the exponential average of $r_\theta$ at step t.
\[
q_{\theta}^{t}\left(y_{<t}, y_{t}\right)=\beta \log \mathbb{E}_{\pi_{ref}\left(y \mid y_{\leq t}\right)} e^{\frac{1}{\beta} r_{\theta}{\left(y\right)}}
\]
\end{proposition}

\begin{proof}
The Proposition can be proven using mathematical induction.
Suppose response $y$ has $T$ steps, then,

\textbf{(1)} For $\forall t < T$, if $q_{\theta}^{t+1}\left(y_{<t+1}, y_{t+1}\right)=\beta \log \mathbb{E}_{\pi_{ref}\left(y \mid y_{\leq t+1}\right)} e^{\frac{1}{\beta} r_{\theta}{\left(y\right)}}$ holds, then $q_{\theta}^{t}\left(y_{<t}, y_{t}\right)=\beta \log \mathbb{E}_{\pi_{ref}\left(y\mid y_{\leq t}\right)} e^{\frac{1}{\beta} r_{\theta}{\left(y\right)}}$ would also hold.

\textbf{(2)} At $t=T$, $q_{\theta}^{T}\left(y_{<T}, y_{T}\right)=r_{\theta}(y)=\beta \log \mathbb{E}_{\pi_{ref}\left(y \mid y_{\leq T}\right)} e^{\frac{1}{\beta} r_{\theta}{\left(y\right)}}$.

\textbf{proof of (1):}
\begin{align*}  
\beta \log \mathbb{E}_{\pi_{ref}\left(y\mid y_{\leq t}\right)} e^{\frac{1}{\beta} r_{\theta}{\left(y\right)}} &= \beta \log \mathbb{E}_{\pi_{ref}\left(y_{t+1}\mid y_{\leq t}\right)} \mathbb{E}_{\pi_{ref}\left(y\mid y_{\leq t+1}\right)} e^{\frac{1}{\beta} r_{\theta}{\left(y\right)}} \\
    &= \beta \log \mathbb{E}_{\pi_{ref}\left(y_{t+1}\mid y_{\leq t}\right)} e^{\frac{1}{\beta} q_{\theta}^{t+1} \left(y_{<t+1}, y_{t+1}\right)} \\
    &= \beta \log \mathbb{E}_{\pi_{ref}\left(y_{t+1}\mid y_{\leq t}\right)} \prod_{i=1}^{t+1} \frac{\pi_{\theta}\left(y_{i} \mid y_{<i}\right)}{\pi_{ref}\left(y_{i} \mid y_{<i}\right)} \\
    &= \beta \log \prod_{i=1}^{t} \frac{\pi_{\theta}\left(y_{i} \mid y_{<i}\right)}{\pi_{ref}\left(y_{i} \mid y_{<i}\right)} \mathbb{E}_{\pi_{ref}\left(y_{t+1}\mid y_{\leq t}\right)} \frac{\pi_{\theta}\left(y_{t+1} \mid y_{\leq t}\right)}{\pi_{ref}\left(y_{t+1} \mid y_{\leq t}\right)} \\
    &=\beta \log \prod_{i=1}^{t} \frac{\pi_{\theta}\left(y_{i} \mid y_{<i}\right)}{\pi_{ref}\left(y_{i} \mid y_{<i}\right)} \sum_{y_{t+1}} \pi_{ref}\left(y_{t+1} \mid y_{\leq t}\right) \frac{\pi_{\theta}\left(y_{t+1} \mid y_{\leq t}\right)}{\pi_{ref}\left(y_{t+1} \mid y_{\leq t}\right)} \\
    &= \beta \log \prod_{i=1}^{t} \frac{\pi_{\theta}\left(y_{i} \mid y_{<i}\right)}{\pi_{ref}\left(y_{i} \mid y_{<i}\right)} \sum_{y_{t+1}} \pi_{\theta}\left(y_{t+1} \mid y_{\leq t}\right) \\
    &= \beta \log \prod_{i=1}^{t} \frac{\pi_{\theta}\left(y_{i} \mid y_{<i}\right)}{\pi_{ref}\left(y_{i} \mid y_{<i}\right)}
\end{align*}

\textbf{proof of (2):}

The conclusion is straightforward. Since $\pi$ is autoregressive, we have

\begin{equation*}
    r_{\theta}(y) : =\beta \log \frac{\pi_{\theta}(y)}{\pi_{ref}(y)}=\beta \log \prod_{i=1}^{T} \frac{\pi_{\theta}\left(y_{i} \mid y_{<i}\right)}{\pi_{ref}\left(y_{i} \mid y_{<i}\right)}=\sum_{i=1}^{T} \beta \log \frac{\pi_{\theta}\left(y_{i} \mid y_{<i}\right)}{\pi_{ref}\left(y_{i} \mid y_{<i}\right)} .
\end{equation*}

Since $y_{\leq T} = y$, the expectation $\mathbb{E}_{\pi_{ref}(y\mid y_{\leq T} )}$ can be removed:

\begin{equation*}
    \beta \log \mathbb{E}_{\pi_{ref}(y \mid y \leq T)} e^{\frac{1}{\beta} r_{\theta}(y)}=\beta \log e^{\frac{1}{\beta} r_{\theta}(y)}=r_{\theta}(y) .
\end{equation*}

\end{proof}

Hence, $q_\theta^t$ represents an exact expectation of outcome reward $r_\theta$ at step $t$, which can be used to calculate process reward for CoT-PRM and KG-PRM.

\section{Hyperparameter Settings and Analysis}

The LLM for generating preference pairs is DeepSeek-R1-Distill-Qwen-32B. For RM baselines, we choose Llama-3.1-8B-Instruct as the reasoning LLM. For our method, we add another LLM, Llama-2-7B-Chat. The parameters of them are frozen. For the LLM only baselines, we use a zero-shot prompt to ask LLMs to answer the questions. 

We train PRMs following OpenRLHF \cite{hu2024openrlhf}. For the training hyperparameters, the maximum learning rate is set to $7 \times 10^{-7}$, and we employ a Warmup-Decay learning rate scheduler \cite{goyal2017accurate}. The maximum sequence length is set to $8192$, and the batch size for training is configured with a macro-batch size of $60$. The Adam optimizer is used with $\beta_1=0.85$ and $\beta_2=0.95$. We optimize the attention computation using FlashAttention \cite{dao2022flashattention} to reduce memory footprint and accelerate training. All experiments are performed on two NVIDIA A100 80GB PCIe GPUs.

\begin{figure}[!htbp] 
    \centering 
    \includegraphics[width=0.4\textwidth]{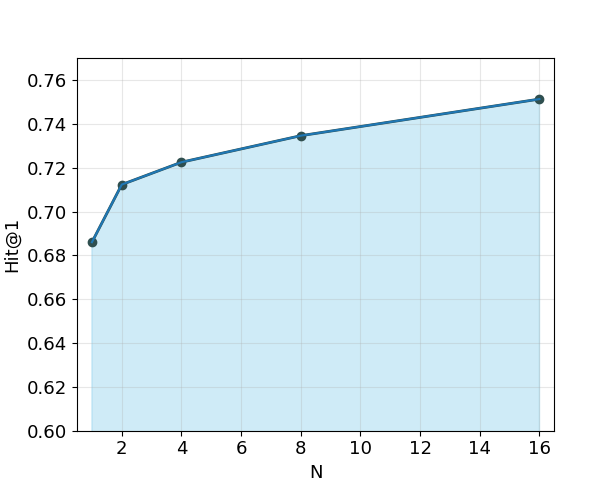} 
    \caption{Performances on different values of $N$.} 
    \label{fig:N} 
\end{figure}

For the number $N$ of Best-of-N (BoN) methods, we conduct experiments on multiple values of it to see how the performance of the model improves. Considering that experiments on the full DPRM are time-consuming and resource-intensive, we conduct experiments on KG-PRM to select the most suitable $N$. The Hit@1 of final answers continues to improve as $N$ increases, as shown in Figure \ref{fig:N}. However, when $N>8$, the increase gradually becomes flat. Considering evaluating computational costs and model efficacy, we select $N=8$ as the optimal value.

To determine the optimal value of $m$ for the top-m matching strategy, we systematically evaluate the model's performance across multiple candidate values of $m$. As shown in Figure \ref{fig:m}, the highest Hit@1 score is achieved when $m=25$, indicating that this setting is the best for our task. Based on this evidence, we fix $m=25$ for all subsequent experiments.

\begin{figure}[!htbp] 
    \centering 
    \includegraphics[width=0.65\textwidth]{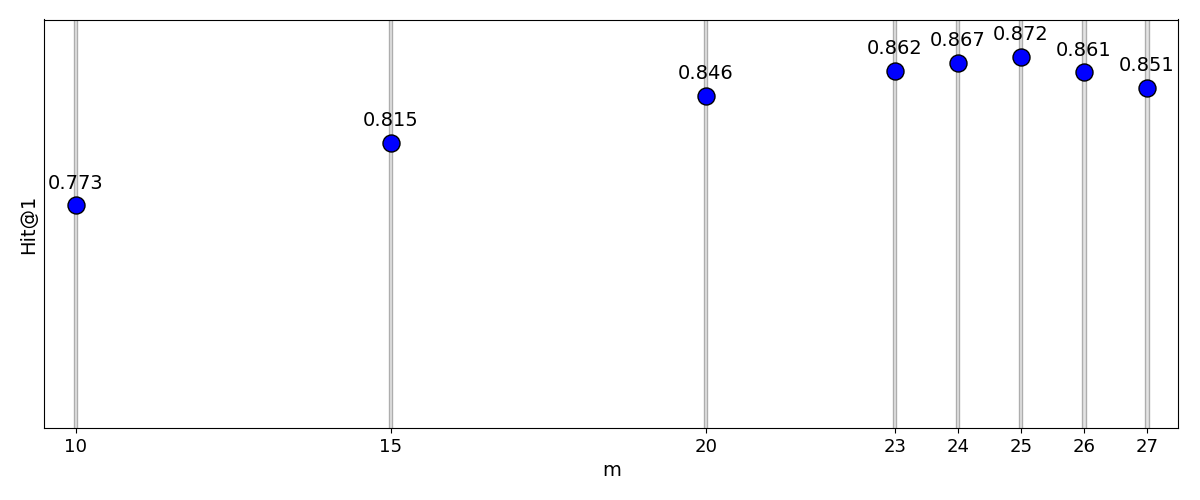} 
    \caption{Performances on different values of $m$.} 
    \label{fig:m} 
\end{figure}

\section{Datasets}

\begin{table}[htbp]  
  \centering  
  \begin{tabular}{cccc}  
    \toprule
    Dataset & \#Train & \#Test & Max \#hop  \\
    \midrule
    WebQSP & 2826 & 1628 & 2\\ 
    CWQ & 27639 & 3531 & 4 \\ 
    \bottomrule
  \end{tabular}
  \caption{Statistics of datasets.}  
  \label{tab:dataset}  
\end{table}

We use two datasets, WebQuestionSP (WebQSP) \cite{yih2016value} and Complex WebQuestions (CWQ) \cite{talmor2018web}. The statistics of the datasets are given in Table \ref{tab:dataset}. We follow previous works \cite{luo2024rog} to use the same train and test splits for fair comparison. The questions in WebQSP are 1-hop or 2-hop, and the questions in CWQ are 2-4 hops. The two datasets test the model's ability to understand and answer questions with multiple facts and reasoning steps. The KG for both datasets is Freebase \cite{bollacker2008freebase}.

\section{Baselines}

The baselines can be divided into four categories: (1) LLM only, (2) KG+LLM, (3) ORM+LLM, (4) PRM+LLM.

(1) LLM-only methods use only LLMs for reasoning without other enhancement methods.

\textbf{Qwen2-7B} \cite{yang2024qwen2technicalreport} is one of a series of LLMs developed by the Alibaba Cloud Tongyi Qianwen team, with a parameter size of 7 billion.

\textbf{Llama-2-7B} \cite{touvron2023llama} is one of the Llama 2 series of LLMs developed by Meta AI, with a parameter size of 7 billion.

\textbf{Llama-3.1-8B} \cite{dubey2024llama} is one of the Llama 3 series of LLMs developed by Meta AI, with a parameter size of 8 billion.

(2) KG+LLM methods use KGs to enhance LLM reasoning.

\textbf{G-Retriever} \cite{he2025g} retrieves the relevant nodes and edges, then constructs the relevant subgraph using the bonus Steiner tree method.

\textbf{GRAG} \cite{hu2025grag} retrieves text subgraphs and performs soft pruning to identify relevant subgraph structures effectively, and proposes a new cue strategy.

\textbf{SubgraphRAG} \cite{li2025simpleeffectiverolesgraphs} generates accurate and explainable answers by efficiently retrieving relevant subgraphs from KGs and leveraging LLMs for reasoning.

\textbf{RoG} \cite{luo2024rog} proposes a planning-search-reasoning framework, which retrieves reasoning paths from KGs to guide LLMs in reasoning.

\textbf{GNN-RAG} \cite{mavromatis2024gnn} integrates graph neural networks (GNNs) as retrieval mechanisms to extract structured knowledge paths from KGs, which are then verbalized and fed into LLMs for answer generation. 

(3) ORM+LLM methods use ORMs to enhance LLM reasoning.

\textbf{SkyworkRM-Llama3.1-8B} \cite{liu2024skywork} is an open-source reward model trained on the high-quality dataset Skywork-SynPref-40M via human-AI collaborative curation.

\textbf{ArmoRM-Llama3-8B} \cite{wang2024armorm} is an open-source multidimensional reward model trained with GPT-4-generated synthetic data via RLAIF.

(4) PRM+LLM methods use PRMs to enhance LLM reasoning.

\textbf{RLHFlow-8B-Mistral-Data} \cite{xiong2024rlhf1} is an open-source LLM alignment framework that trains 8B parameter models via DPO (Direct Preference Optimization) using Mistral-generated synthetic preference data.

\textbf{RLHFlow-8B-DeepSeek-Data} \cite{dong2024rlhf2} extends the RLHFlow framework by leveraging DeepSeek-generated preference pairs for DPO training.

\textbf{ImplicitPRM-DPO} \cite{yuan2024implicitprm} pioneers annotation-free process reward modeling by combining implicit step-level reward parameterization with DPO, enabling efficient multi-step reasoning alignment without human/AI-generated rollouts.

\section{Prompt Templates}
\label{prompt}

\begin{figure}[!htbp] 
    \centering 
    \caption{Prompt templates for generating CoT-PRM training data.}
    \label{fig:ge-kg} 
    
    \begin{subfigure}{0.8\textwidth}
        \centering
        \includegraphics[width=\linewidth]{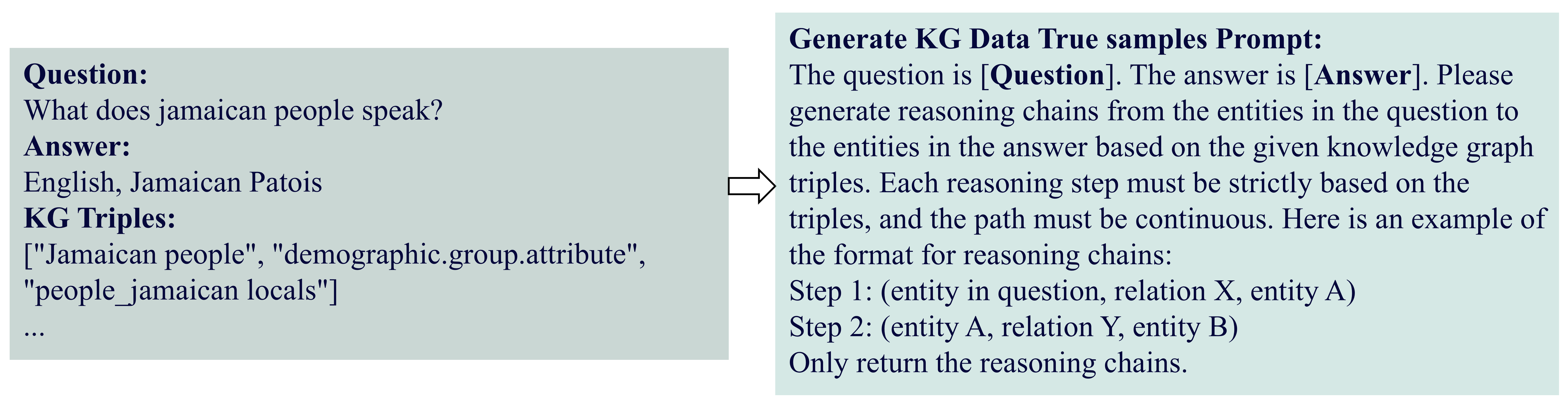} 
        \caption{Prompt template for true samples.}
    \end{subfigure}
    \hfill 
    \begin{subfigure}{0.8\textwidth}
        \centering
        \includegraphics[width=\linewidth]{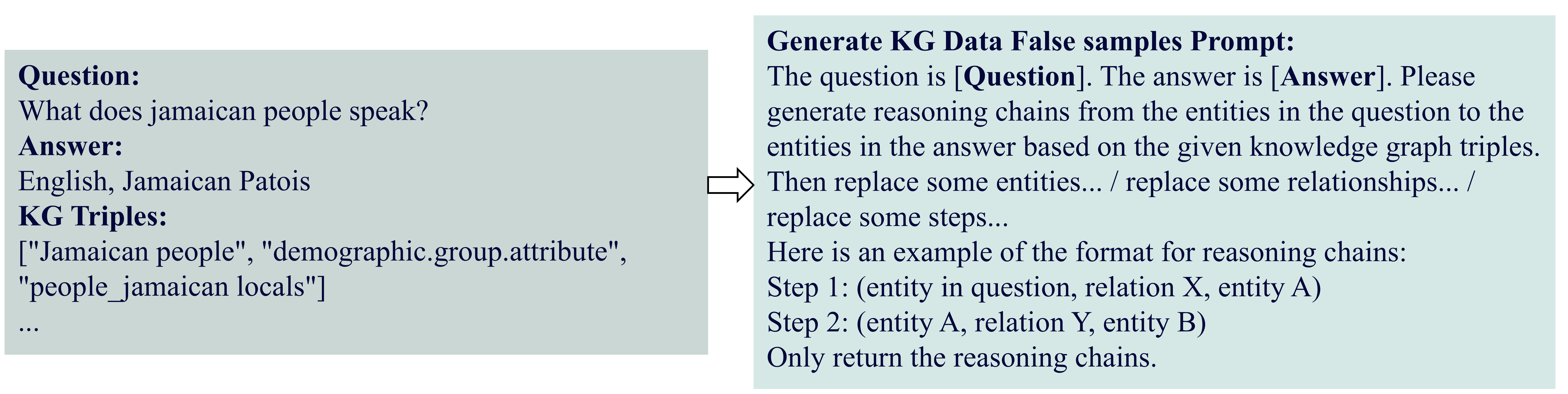} 
        \caption{Prompt template for false samples.}
    \end{subfigure}
\end{figure}

For the prompt to generate training data for KG-PRM, the prompt templates are in Figure \ref{fig:ge-kg}
The true samples in the data are KG paths from the question entities to the ground truth entities. And the false samples are obtained in the following ways: (1) Factual errors (The entities in triples are wrong.); (2) Logical errors (The relations in triples are wrong.); (3) Logical break (The entities in adjacent triples of a KG path cannot be aligned or connected.).

\begin{figure}[!htbp] 
    \centering 
    \caption{Prompt templates for generating KG-PRM training data.}
    \label{fig:ge-cot} 
    
    \begin{subfigure}{0.8\textwidth}
        \centering
        \includegraphics[width=\linewidth]{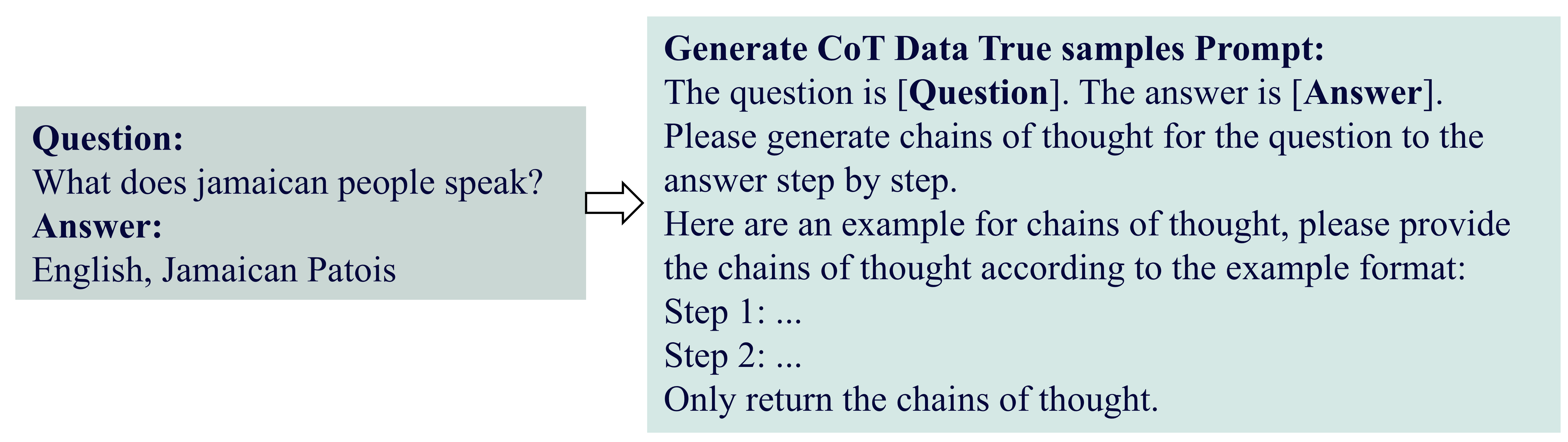} 
        \caption{Prompt template for true samples.}
    \end{subfigure}
    \hfill 
    \begin{subfigure}{0.8\textwidth}
        \centering
        \includegraphics[width=\linewidth]{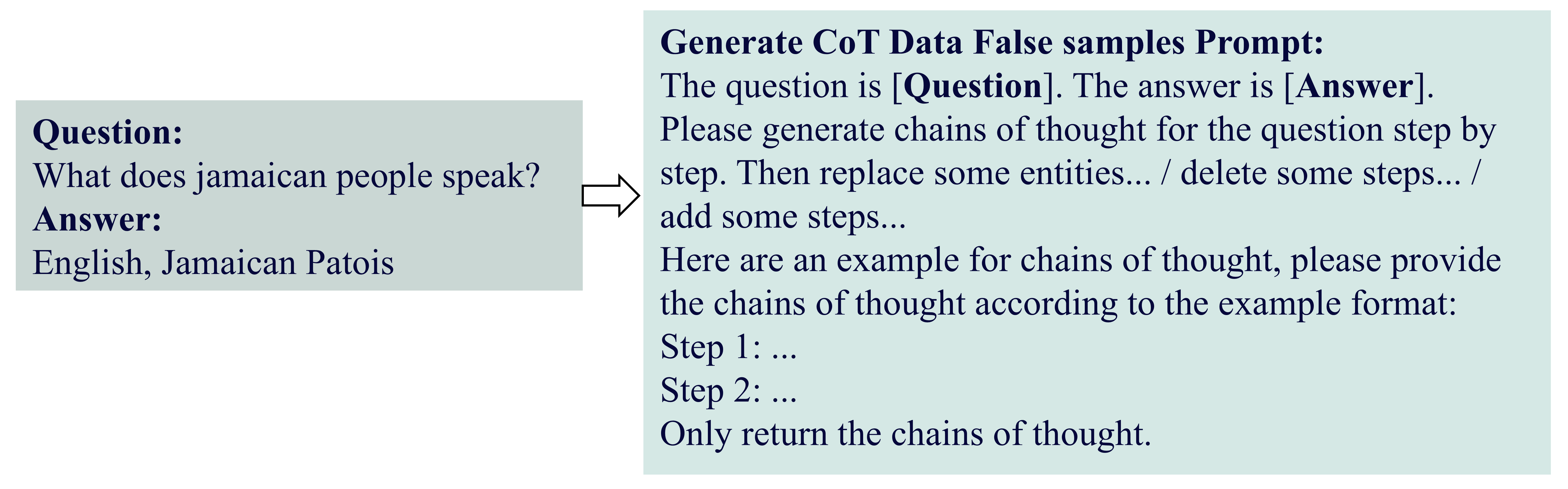} 
        \caption{Prompt template for false samples.}
    \end{subfigure}
\end{figure}

For the prompt to generate training data for CoT-PRM, the prompt templates are in Figure \ref{fig:ge-cot}
The true samples obtain the ground truth, and the false samples are obtained in the following ways: (1) Factual errors (wrong content in reasoning steps); (2) Step skipping (omitting some reasoning steps). (3) Step redundancy (Some steps are unrelated to the QA pairs.).

\begin{figure}[!htbp] 
    \centering 
    \includegraphics[width=0.8\textwidth]{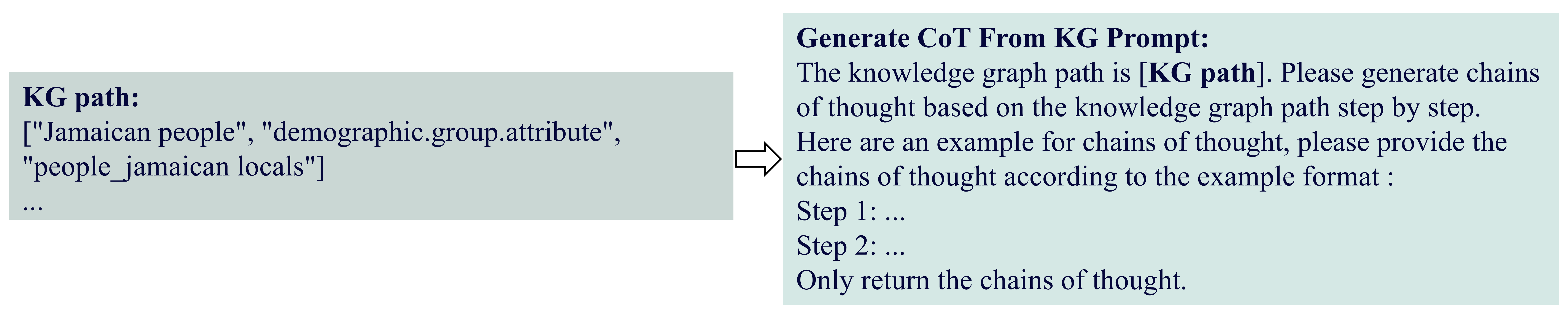} 
    \caption{Prompt template for KG-derived CoT data.} 
    \label{fig:cot2} 
\end{figure}

For the prompt to transform KG samples to CoT samples, we convert the triples in samples into continuous texts, as shown in Figure \ref{fig:cot2}.

\begin{figure}[!htbp] 
    \centering 
    \includegraphics[width=0.8\textwidth]{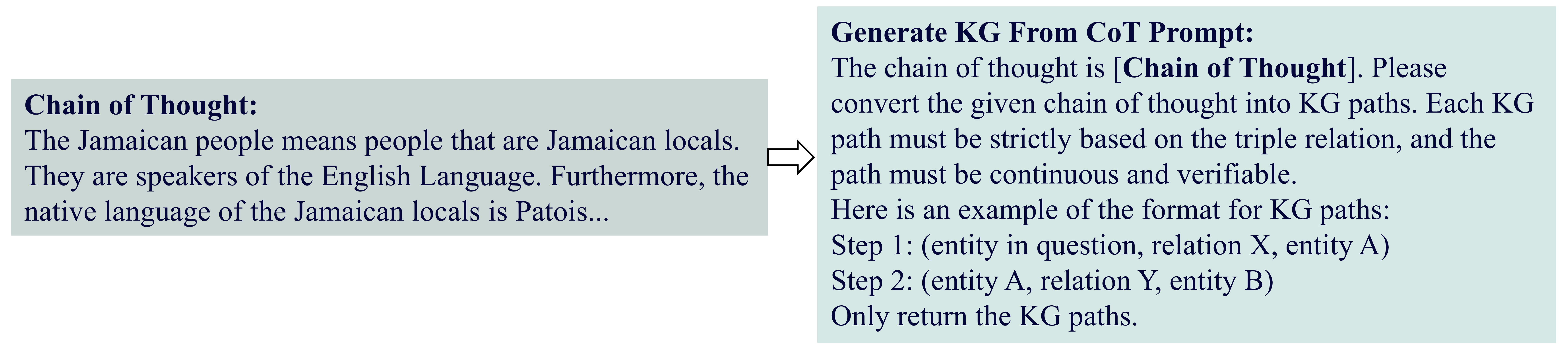} 
    \caption{Prompt template for CoT-derived KG data.} 
    \label{fig:kg2} 
\end{figure}

For the prompt to transform CoT samples to KG samples, we extract triples in each step and connect them into entity-relation paths, as shown in Figure \ref{fig:kg2}.

\begin{figure}[!htbp] 
    \centering 
    \includegraphics[width=0.8\textwidth]{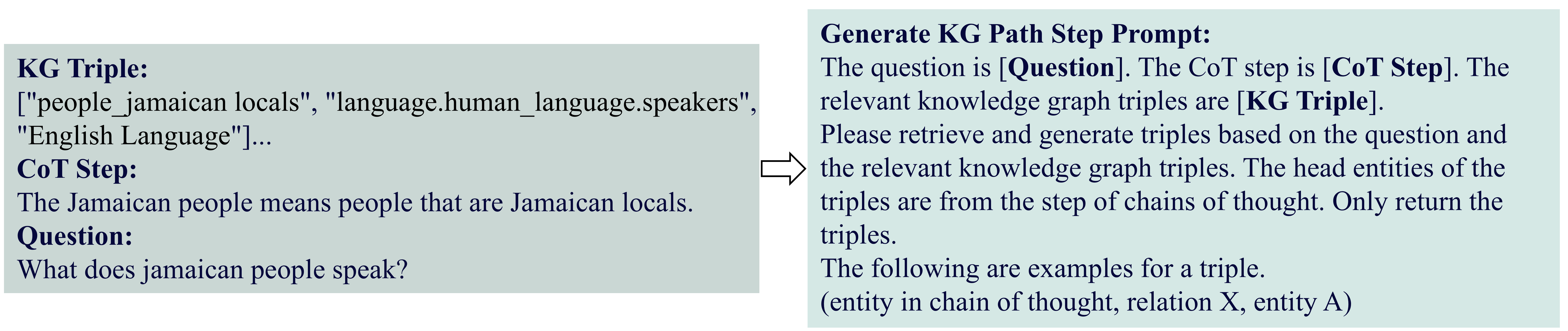} 
    \caption{Prompt template for obtaining KG path steps.} 
    \label{fig:promptkg} 
\end{figure}

For the prompt to retrieve and reconstruct the KG path steps, the LLM retrieves triples from the triple set $tr_i$ and reconstructs them according to the preset prompt template to make the triple start from the entities in the previous CoT step $c_{i-1}$, as shown in Figure \ref{fig:promptkg}.  

\begin{figure}[!htbp] 
    \centering 
    \includegraphics[width=0.8\textwidth]{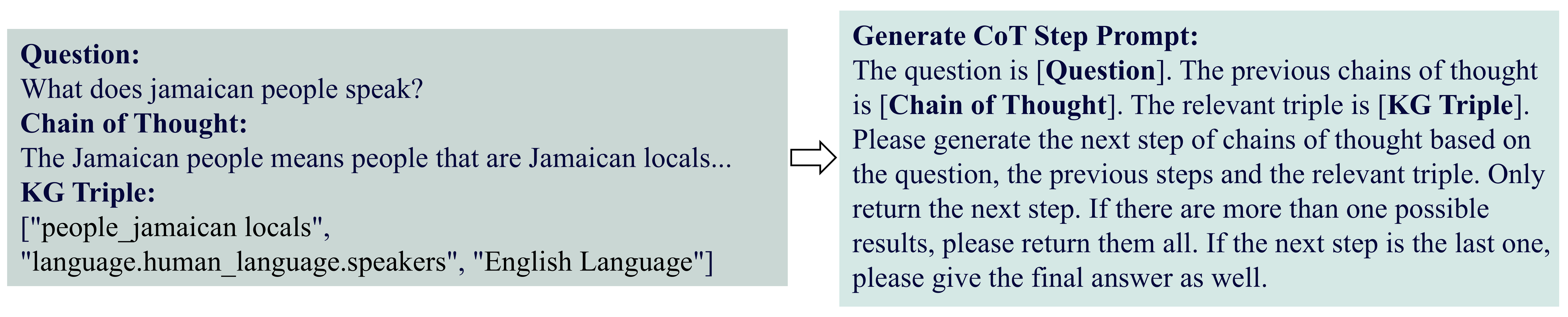} 
    \caption{Prompt template for obtaining CoT steps.} 
    \label{fig:promptcot} 
\end{figure}

For the prompt to generate the CoT steps, the KG step $p_{i}$ is first converted into a sentence and then combined with the question $Q$ and the previous CoT ${c_1,\ldots, c_{i-1}}$ according to the preset prompt template, as shown in Figure \ref{fig:promptcot}. The prompt is then submitted to the LLM to generate new CoT steps.

\begin{figure}[!htbp] 
    \centering 
    \includegraphics[width=0.8\textwidth]{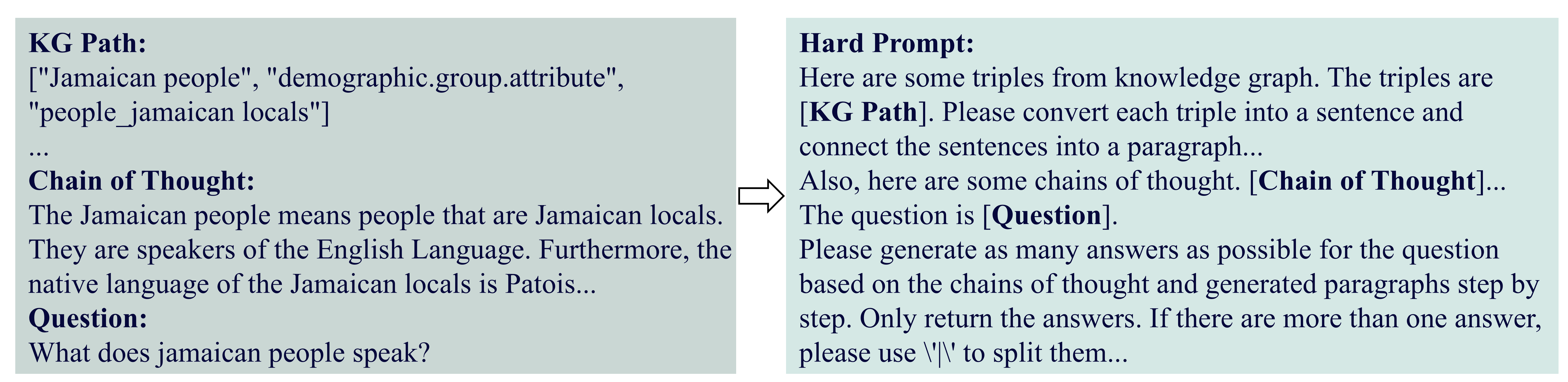} 
    \caption{Hard prompt template.} 
    \label{fig:hard prompt} 
\end{figure}

For the hard prompt, the triplets in the KG paths $p$ are first converted into sentences and then combined with the question $Q$ and the CoTs $c$ according to the preset prompt template, as shown in Figure \ref{fig:hard prompt}.

\begin{figure}[!htbp] 
    \centering 
    \includegraphics[width=0.8\textwidth]{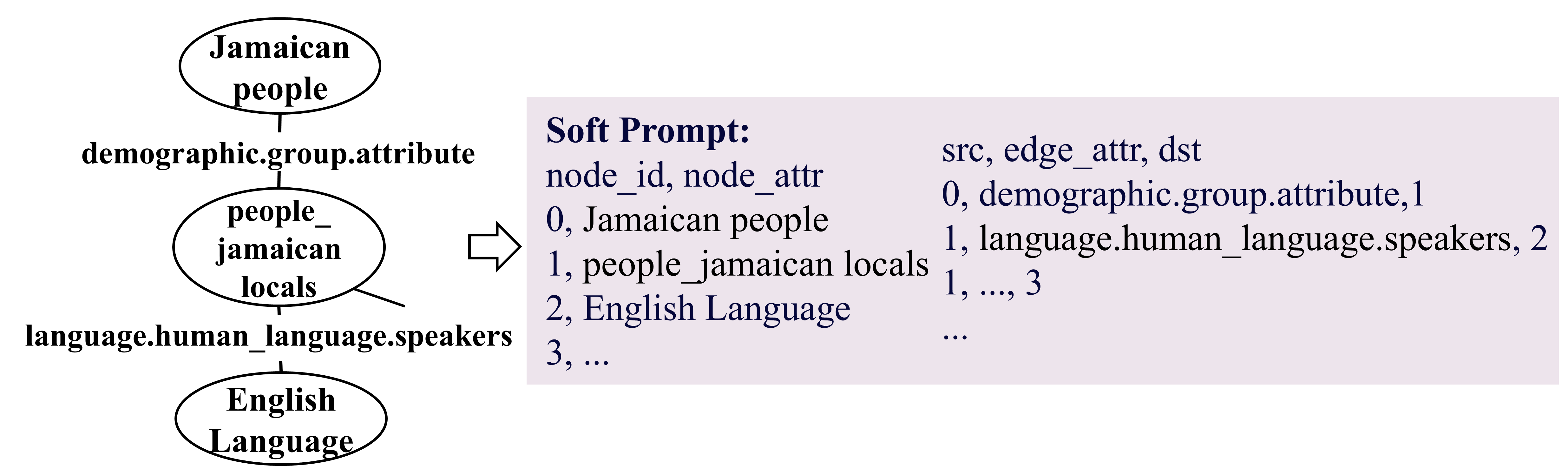} 
    \caption{Soft prompt template.} 
    \label{fig:soft prompt} 
\end{figure}

For the soft prompt, the graph structure of KG paths $p$ is textualized, as shown in Figure \ref{fig:soft prompt}.



\end{document}